\documentclass{article}
\usepackage{graphicx} % Required for inserting images

\usepackage[T1]{fontenc}
\usepackage{graphicx}
\usepackage{booktabs}
\usepackage[misc]{ifsym}
\usepackage[margin=0.9in]{geometry}
\usepackage{amsmath,amssymb,dsfont,mathtools,amsthm}
\usepackage{xcolor}
\usepackage{soul}
\usepackage[ruled]{algorithm2e}
\usepackage[numbers]{natbib}
\usepackage{pgfplots}
\usepackage{tikz}
\usetikzlibrary{shapes.symbols}
\usepackage{bm}
\usepackage{subfigure}
 
\pgfplotsset{compat=1.18}

\newtheorem{thmprop}{Proposition}

\newtheorem{thmasmp}{Assumption}

\newtheorem*{thmrem*}{Remark}
\newtheorem*{thmprop*}{Proposition}

\newcommand{\red}[1]{\textcolor{red}{#1}}
\newcommand{\blue}[1]{\textcolor{blue}{#1}}
\newcommand{\green}[1]{\textcolor{green}{#1}}

\newcommand{\orange}[1]{\textcolor{orange}{#1}}

\def\E{\mathbb{E}}

\def\bx{{\bm{x}}}

\def\bX{{\bm{X}}}

\def\bbR{\mathbb{R}}

\def\cH{\mathcal{H}}

\def\cX{\mathcal{X}}
\def\cY{\mathcal{Y}}

\begin{document}
%%% want to do submission to ECML workshop https://wafl2025.di.unito.it/
%%% Deadline 14/6

% Want to have more complete results then, add ablation with more local epochs and disagreement for all methods

% Either short or long paper (6 or 12 pages + references), decide when results are more clear

\title{Federated Learning with Heterogeneous and Private Label Sets}
\author{
Adam Breitholtz\footnotemark[1] \\
Chalmers University of Technology\\
\& University of Gothenburg \\
\texttt{adambre@chalmers.se}
\and
Edvin Listo Zec\thanks{Equal contribution.}\\
RISE Research Institutes of Sweden\\ KTH Royal Institute of Technology\\
\texttt{edvin.listo.zec@ri.se} 
\and 
Fredrik D. Johansson\thanks{ https://www.healthyai.se/} \\
Chalmers University of Technology\\
\& University of Gothenburg \\
\texttt{fredrik.johansson@chalmers.se} \\
}

\maketitle              % typeset the header of the contribution

\begin{abstract}
%The abstract should briefly summarize the contents of the paper in
%150--250 words.
Although common in real-world applications, heterogeneous client label sets are rarely investigated in federated learning (FL). Furthermore, in the cases they are, clients are assumed to be willing to share their entire label sets with other clients. Federated learning with   \emph{private} label sets, shared only with the central server, adds further constraints on learning algorithms and is, in general, a more difficult problem to solve. In this work, we study the effects of label set heterogeneity on model performance, comparing the public and private label settings---when the union of label sets in the federation is known to clients and when it is not. We apply classical methods for the classifier combination problem to FL using centralized tuning, adapt common FL methods to the private label set setting, and discuss the justification of both approaches under practical assumptions. Our experiments show that reducing the number of labels available to each client harms the performance of all methods substantially. Centralized tuning of client models for representational alignment can help remedy this, but often at the cost of higher variance. Throughout, our proposed adaptations of standard FL methods perform well, showing similar performance in the private label setting as the standard methods achieve in the public setting. This shows that clients can enjoy increased privacy at little cost to model accuracy.

%Our experiments show that varying the number of labels available to each client substantially affects the performance of all methods and that the adaptations of standard FL methods perform well, even in the private label set setting. We find that centralized tuning of client models for representational alignment can improve performance; however, this is not guaranteed in general. 

%
% We try to consider the issue with missing labels between clients, in particular how the privacy of the label sets of clients may impact performance. 
%
% Private label sets - clients are unaware of what they are learning, they just learn from the classes they have.

% Deng et al considers the case where the clients have some unknown categories which they still have X's for. This is akin to UDA but at a client level.

% 

\end{abstract}

\section{Introduction}
% Sparsity in labels generally
%In federated learning (FL), the training efforts of multiple clients are aggregated by a central server to form a stronger model than either client could produce on their own, all without clients sharing their private data. In standard federated classification, clients are assumed to use class labels from the same discrete set $\cY$, e.g., a set of animals captured in photographs or categories of products recorded in customer sales logs. However, in applications, it is common for some clients never to observe some of the classes. For example, some animal species may not inhabit a geographical region represented by a client or a client store that does not stock a particular item for sale. 

Federated learning (FL) enables collaborative model training across distributed clients without centralizing their private data \citep{mcmahan2017communication}. While promising, the effectiveness of FL is often challenged by statistical heterogeneity, where the data distributions vary significantly across clients. A particularly common and disruptive form of this is \emph{label shift}~\citep{li2021fedrs}, where the distribution of class labels differs from one client to another. Even more challenging is \emph{label set heterogeneity}~\citep{huang2024fedcrl}, where clients' local label sets are disjoint subsets of a global label set. 
%
%For example, in a medical imaging consortium, one hospital may have data for specific rare diseases that others have never encountered.
%
% private label sets
%In an even more challenging case, clients may prefer to keep their label set private from other clients, divulging it only to the central server. This could be the case in defense applications, allied nations might want to learn a model to classify forces on satellite images collaboratively while being uneasy about sharing the specifics about the labeled information they have available to them. Thus a private label set setup could be used where the server-side application is inspected by both parties before initialization. 
%
%In this work, we investigate a challenging and practical variant of this problem:   %\hltodo{comment shift in x?}
%
%We further introduce the \emph{private} label set setting, 
%

In applications where access to instances of particular classes holds a competitive advantage, clients may be unwilling to reveal the identities of the classes they observe. Consider, for instance, a consortium of competing pharmaceutical companies wanting to train a model to predict which drug compounds individual patients will have adverse reactions to~\citep{edwards2000adverse,lavan2016predicting}. Each company has proprietary data on its own set of compounds and reactions, some of which are used by other clients in the federation, and some which are not. They are willing to collaborate to build a more powerful, generalizable model, but would never share the full list of compounds they classify with other clients, as this would reveal information about their R\&D pipeline. Instead, clients must communicate model updates with the central server pertaining  \emph{only} to their  \emph{private label set}.

% Defense example: Consider members of a defense alliance who want to train a model for predicting coordinates, countries or similar from satellite images. They would be willing to collaborate to get a better, more general model. However, they may not want others to know which specific locations they surveil and would thus like the labels to be private.

Learning with heterogenous client label distributions has been tackled by several methods, including model distillation~\citep{gudur2020federated}, contrastive learning~\citep{huang2024fedcrl}, and latent space alignment based on class names~\citep{zhang2023navigating}. However, the literature is more sparse when considering label sets which are not identical across clients, although there are some works which consider it~\cite{li2021fedrs}. Moreover, classical methods for federated learning, adapted for label shift or otherwise, can not be applied directly with private label sets. In this setting, the models learned on the clients will necessarily be incompatible for regular aggregation since they will make classifiers for different label sets. Therefore, there is a need to develop methods to deal with this complication to ensure that learning is successful.

% complications from not sharing label sets

% similarity to distribution shift/ label heterogeneity
% This setting is related to other problems which deal with other kinds of distributional shifts, e.g., covariate shift in domain generalization. %There are many methods in FL which seek to adress these kinds of distributional imbalances as it often arises 

In this work, we investigate the effects of client label sparsity and heterogeneity on federated learning performance when client label sets are shared by the whole federation (public) and when they are unknown to other clients (private). We define the private label set problem in Section~\ref{sec:problem} and adapt popular FL model aggregation strategies for it in Section~\ref{sec:private}. We show that such methods are not well justified when client representations are poorly aligned and propose an alternative method based on the literature on classifier combination in Section~\ref{sec:alignment}, tuning the central classifier for heterogeneous client representations using an unlabeled dataset at the server. We conduct experiments in image classification on two data sets and show empirically how the sparsity and privacy of label sets affect performance~(Section \ref{sec:experiments}). We find that both the private and public label settings are more challenging when clients hold smaller and more diverse subsets of the global label set. Finally, in the private label setting, our proposed adaptations achieve comparable performance to methods for public labels, implying that clients can retain more privacy at little to no cost in accuracy.

% \subsection{Main Contributions}
 
% Please note that the first paragraph of a section or subsection is
% not indented. The first paragraph that follows a table, figure,
% equation etc. does not need an indent, either.

% Subsequent paragraphs, however, are indented.

% \tikzstyle{mybox} = [draw=cyan, fill=blue!20, very thick,
%     rectangle, rounded corners, inner sep=10pt, inner ysep=10pt]
\tikzstyle{fancytitle} =[minimum height= 0.75cm, fill=blue, text=white, ellipse]
\def\bar#1#2#3{\addplot [ybar, bar width=#3 cm, draw=black, 
        fill=#1
    ] coordinates {
        #2
    };} % color, coord, width
    \begin{figure}[t!]
        \label{fig:setup}
        \centering 
\hspace*{-2cm}\subfigure[\hspace*{-2cm}]
{%
 \begin{tikzpicture} %\textbf{\vspace*{0.5cm}\hspace*{-4cm}}

\node[cloud, minimum height=0.5cm,minimum width=2cm,draw] (b) at (-3.5,0.75) {Server};

\node[minimum width=1cm,draw,circle] (c) at (-0.5,0) {$S_i$};
\node[minimum width=1cm,draw,circle] (a) at (-0.5,-1.5) {$S_j$};
\node[] (theta) at (-1.5,0.5) {$\theta_{t,i}$};
\node (d) at (-0.5,0.85) {$\vdots$};
\node (f) at (-0.5,-0.65) {$\vdots$};
\node[minimum width=1cm,draw,circle] (e) at (-0.5,1.5){$S_1$};
\node[anchor=south] (hist1_pub) at (1.25,0.9) { \begin{axis}[
    symbolic x coords={ -0.1, 0, 1, 2, 3, 4, 4.1},
    ymajorticks=false,
    xticklabels={,,},
extra x ticks={0,2,4}, 
extra x tick label={\null},
     xmin=-0.1,
    xmax=4.1,
     ymin=0,
    area style,
    width=3cm]
     \bar{black}{(0,0)}{0.2}
     \bar{red}{(1,40)}{0.2}
       \bar{blue}{(2,15)}{0.2}
        \bar{orange}{(3,10)}{0.2}
         \bar{green}{(4,5)}{0.2}
          % \bar{magenta}{(5,10)}{0.2}
          % \bar{brown}{(6,15)}{0.2}
\end{axis}};
\node[anchor=south] (histi_pub) at (1.25,-0.55) { \begin{axis}[
    symbolic x coords={ -0.1, 0, 1, 2, 3, 4, 4.1},
    ymajorticks=false,
    xticklabels={,,},
extra x ticks={0,2,4}, 
extra x tick label={\null},
     xmin=-0.1,
    xmax=4.1,
     ymin=0,
    area style,
    width=3cm]
     \bar{black}{(0,30)}{0.2}
     \bar{red}{(1,0)}{0.2}
       \bar{blue}{(2,20)}{0.2}
        \bar{orange}{(3,18)}{0.2}
         \bar{green}{(4,0)}{0.2}
          % \bar{magenta}{(5,10)}{0.2}
          % \bar{brown}{(6,15)}{0.2}
\end{axis}};
\node[anchor=south] (histj_pub) at (1.25,-2.1) { \begin{axis}[
    symbolic x coords={ -0.1, 0, 1, 2, 3, 4, 4.1},
     ymajorticks=false,
    xticklabels={,,},
extra x ticks={0,2,4}, 
extra x tick label={\null},
     xmin=-0.1,
    xmax=4.1,
    ymin=0,
    area style,
    width=3cm]
     \bar{black}{(0,5)}{0.2}
     \bar{red}{(1,25)}{0.2}
       \bar{blue}{(2,40)}{0.2}
        \bar{orange}{(3,0)}{0.2}
         \bar{green}{(4,15)}{0.2}
             
\end{axis}};
\node[anchor=south] (histi) at (-0.4,-0.55) { \begin{axis}[
    symbolic x coords={ -0.1, 0, 1, 2, 2.1},
     ymajorticks=false,
    xticklabels={,,},
extra x ticks={0,1,2}, 
extra x tick label={\null},
    xmin=-0.1,
    xmax=2.1,
    ymin=0,
    area style,
    width=3cm]
          \bar{blue}{(0,20)}{0.28}
          \bar{orange}{(1,18)}{0.28}
          \bar{black}{(2,30)}{0.28}
\end{axis}};
\node[anchor=south] (histj) at (-0.5,-2.1) { \begin{axis}[
    symbolic x coords={-0.1, 0, 1, 2, 3, 3.1},
    ymajorticks=false,
    xticklabels={,,},
extra x ticks={0,2,3}, 
extra x tick label={\null},
     xmin=-0.1,
    xmax=3.1,
     ymin=0,
    area style,
    width=3cm]
          \bar{green}{(0,15)}{0.22}
          \bar{black}{(1,5)}{0.22}
          \bar{blue}{(2,40)}{0.22}
          \bar{red}{(3,25)}{0.22}
\end{axis}};
\node[anchor=south] (hist1) at (-0.5,0.9) { \begin{axis}[
    symbolic x coords={-0.1, 0, 1, 2, 3,3.1},
     ymajorticks=false,
    xticklabels={,,},
extra x ticks={0,2,3}, 
extra x tick label={\null},
     xmin=-0.1,
    xmax=3.1,
     ymin=0,
    area style,
    width=3cm]
       \bar{green}{(0,5)}{0.22}
        \bar{red}{(1,40)}{0.22}
          \bar{orange}{(2,10)}{0.22}
          \bar{blue}{(3,15)}{0.22}
\end{axis}};

\node[anchor=south] (histserver) at (-5.5,-2) { \begin{axis}[
    symbolic x coords={ -0.1, 0, 1, 2, 3, 4, 4.1},
    %ticks=none,
    ymajorticks=false,
    xticklabels={,,},
extra x ticks={0,2,4}, 
extra x tick label={\null},
     xmin=-0.1,
    xmax=4.1,
     ymin=0,
    area style,
    width=3.45cm]
     \bar{black}{(0,40)}{0.25}
     \bar{red}{(1,40)}{0.25}
       \bar{blue}{(2,25)}{0.25}
        \bar{orange}{(3,35)}{0.25}
         \bar{green}{(4,30)}{0.25}
          % \bar{magenta}{(5,10)}{0.2}
          % \bar{brown}{(6,15)}{0.2}
\end{axis}};
\node (priv) at (1.25, -2.3) {\textbf{Private}};
\node (pub) at (3.1, -2.3) {\textbf{Public}};
% \node[minimum width=2cm,draw,rectangle, anchor=west] (l1) at (1,1.35) {$\mathcal{L}_1=\{0,1,2,3\}$};
% \node[minimum width=2cm,draw,rectangle, anchor=west] (li) at (1,-0.45) {$\mathcal{L}_i=\{0,1,2\}$};
% \node[minimum width=2cm,draw,rectangle, anchor=west] (lj) at (1,-2.35) {$\mathcal{L}_j=\{0,1,2,3\}$};
 \node[minimum height=0.75cm,draw,rectangle] (h) at (-0.5,2.75) {~$( X_{tune}^i)_{i=1}^M$, $\cY=\{cat,\red{dog},\blue{bird},\orange{monkey},\green{pig}\}$~};
\draw[<->] (b) -- (a);
\draw[<->] (b) -- (c);
\draw[<->] (b) -- (e);
\draw[<->,] (b) -- (e);
\draw[->] (h) to[out=-135, in=135, looseness=1] (b);
% \node[] (bars) at (7.5,0) {\includegraphics[width=0.55\textwidth]{images/barplot_cifar_5.pdf}};
\end{tikzpicture} 
}
\subfigure[]
    {
    \includegraphics[width=0.7\textwidth]{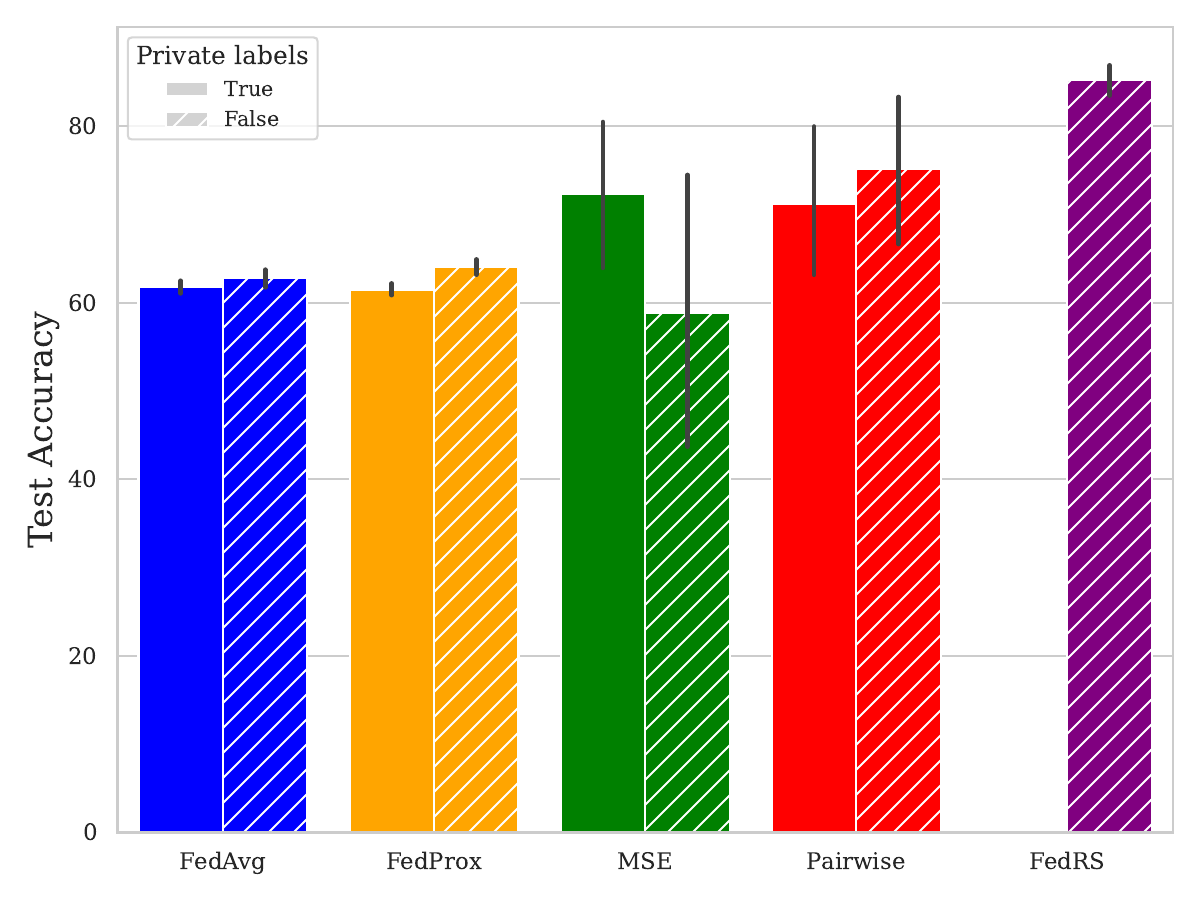}
    }
     \caption{\textbf{a): }A schematic view of the two settings which we consider in our work. The private setting where the clients are unaware of the full label set and the public setting where this is known. \textbf{b): }Results on CIFAR10 where each client has 5 labels available in their respective dataset. The tuning methods with MSE and Pairwise losses perform the best in the private setting. Errorbars represent a $95\%$ confidence interval. Note, FedRS is not applicable with private labels as it needs knowledge of the full label set.}
\end{figure}
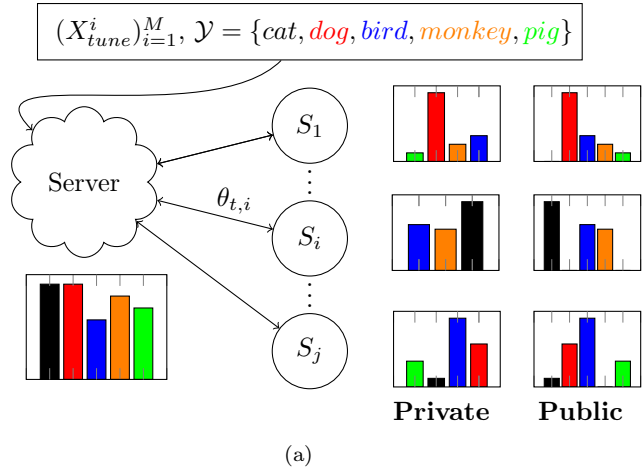
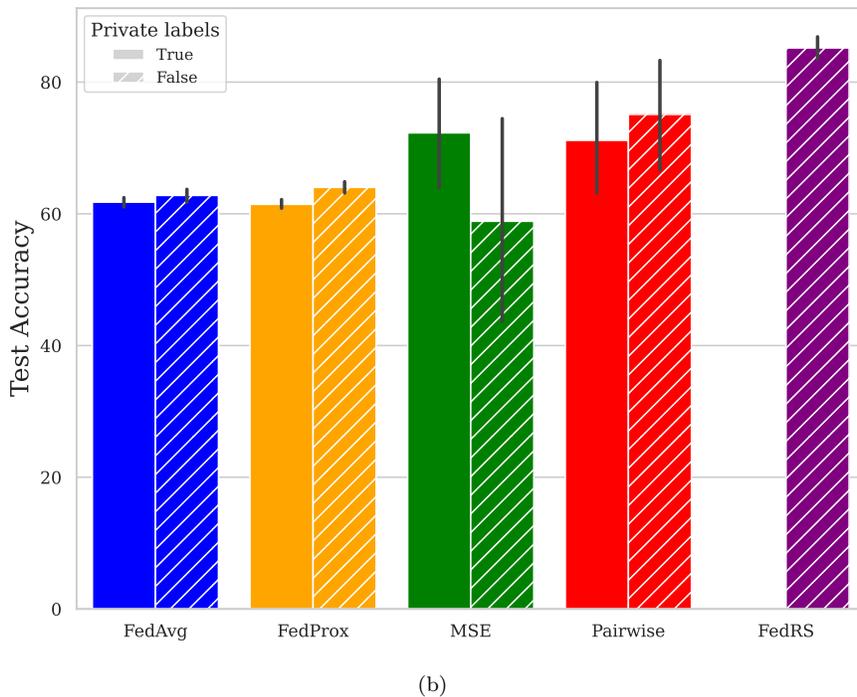%
\section{Related Work}
The question of how to combine classifiers trained on disparate label sets have been studied in the binary setting previously in the context of centralized (non-federated) learning. See for example~\cite{wu2003probability} for a comparison of several methods which focuses on combining  binary classifiers based on the classifier probabilites. 

%% open set stuff
A similar line of work is the Open-set literature, where the label sets in the clients may be incomplete and the sets may not match between clients. In \cite{Dengmismatch23} they learn distribution estimators in the clients to approximate the overall label distribution using uncertainty of the global model. 

\cite{zhang2023navigating} deals with clients which do not share the same label set. They propose having the clients share the names of the labels and aligning the embedding of these names across the representations of clients. However, this requires sharing the names of the labels which may be undesirable, especially in a private label set setting. In a similar vein, \cite{li2021fedrs} restricts the softmax to account for the missing labels in the clients. The algorithm proposed hinges on knowing which labels a client has to account for which precludes its use in the private label set setting. % is this true?

Moreover, our setting is also related to shift in client label distributions. There are many works which aim to handle cases when there is a imbalance between client label distributions. This can be done by using regularization which penalizes large deviations in client updates~\cite{fedprox, li2021ditto} or using control variates to steer the learning~\cite{scaffold}. Other techniques include clustering clients with similar data distributions and training separate models for each cluster~\citep{ghosh2020efficient, sattler2020clustered, vardhan2024improved} and meta-learning to enable models to quickly adapt to new data distributions with minimal updates~\citep{chen2018federated, jiang2019improving}. However, these techniques are not adapted to the private label set setting.

Another related field is that of semi supervised federated learning where some works make use of unlabeled data in FL settings~\cite{jeong2021federated}. In these works the unlabeled data is usually available on the client side, which differs from our setting. 
% FJ: I think \citet{zhang2023navigating} is the closest to our setting. However, they assume that the clients share the names of all of their labels, which they may not want. Their method is also a complete heuristic---their theory does not explain why models for subsets of labels can be combined. 

\section{Problem setup}
\label{sec:problem}
We consider the problem of federated learning (FL) of a single  model $h$, trained to classify points $\bx \in \cX$ into classes $y \in \cY = \{1, ..., |\cY|\}$ given inputs $\bx \in \cX \subseteq \bbR^d$. A central server coordinates learning on $m$ clients, indexed by $k=1, ..., m$, each observing labeled data from an unknown distribution $p_k(\bX, Y)$. Central to our setting is that clients \emph{do not} have all labels in their data, i.e., clients $k$ are exposed only to a subset $\cY_k\subset \cY$ of labels.

Our goal is to learn a probabilistic classifier $h : \cX \rightarrow \Delta_\cY$, where $\Delta_\cY$ is the simplex over \emph{all} classes, that minimizes the expected prediction risk with respect to a loss function $L : \Delta_\cY \times \cY \rightarrow \bbR$,
\begin{equation}\label{eq:risk}
  \underset{h}{\text{minimize}} \;R(h) \quad R(h) \coloneqq \E[L(h(\bX), Y)]~.
\end{equation}
Here, the expectation $\E$ is defined over an unknown distribution $p(\bX, Y)$, assumed to be a convex combination of clients $p(\bX, Y) = \sum_{k=1}^m w_k p_k(\bX, Y)$ where $w \in \Delta_m$ assigns a weight to each client. In the classical FL setting, it is assumed implicitly that the weight is proportional to the number of samples $n_k$ held by the client, $w_k = n_k/(\sum_{k'=1}^m n_{k'})$, and the risk is computed over a distribution that matches the client aggregate, as exploited in the aggregation scheme of federated averaging (FedAvg)~\citep{mcmahan2017communication}. However, our methods can be adapted to targeted federated learning where $p$ cannot be expressed by a convex combination of clients~\citep{zec2024overcoming}.
The model $h(x)=\psi(\phi(x))$ typically consists of feature extractor $\phi$ and a classifier $\psi$, typically parameterized by a neural network with parameters $\theta_\phi$ and a linear-softmax classifier with parameters $\theta_\psi$, respectively. That is, $h_\theta(x) = \sigma(\theta_\psi ^\top \phi(x))$ where $\theta = (\theta_\phi, \theta_\psi)$.

We will consider two settings for the label set (illustrated in Figure \ref{fig:setup}): \begin{itemize} 
\item \textbf{Public labels}: All clients know the full global label set $\cY$. This is the standard FL setting but with emphasis on label set heterogeneity. 
\item \textbf{Private labels}: Each client know only their local label set $\cY_k$ and all communication with the central server is restricted to this set. That is, classifier parameters $\theta_\psi(y)$ for labels $y \not\in \cY_k$ are not shared with client $k$.
\end{itemize}
In both cases, the central server knows the full label set and the label sets of all clients to allow for tailoring communication to clients with heterogeneous and private label sets. We expand on methods for handle the private label case next. 

%
% METHOD
%
\section{Methods for heterogeneous label sets}
\label{sec:method}
When the label set is public, we can simply use existing FL methods to learn our classifiers. The issue of label set heterogeneity still remains, and aligning the models to combat any effects of misaligned representations may be warranted. However, in the case where the label set is private, some further modifications have to be made. We detail this and a method of tuning models for alignment in the following sections. % \hltodo{What about dealing with labels et heterogeneity? Still relevant, no?}

%% private label set
\subsection{Model averaging with private label sets}
\label{sec:private}
The main challenge addressed in this work is \emph{private label set heterogeneity}: each client $k$ observes labels from a subset $\cY_k \subseteq \cY$ and are \emph{unaware} of other labels $\cY \setminus \cY_k$. Without loss of generality, we assume that for all clients $k$, every label in $\cY_k$ is observed with positive probability, $\forall y \in \cY_k : p_k(Y=y) > 0$, and other labels are unobserved, $\forall y \not\in \cY_k : p_k(Y=y) = 0$.

In the private setting, standard methods (e.g., FedAvg~\citep{mcmahan2017communication}, FedProx~\citep{fedprox}, FedRS~\citep{li2021fedrs}) cannot be applied without modification as clients do not have access to the full set of parameters $\theta_\psi$ of the shared classifier $\psi$. Moreover, the server can only receive updates from client $k$ to parameters concerning their subset of labels $\cY_k$. To overcome this obstacle, we propose a simple modification to common model averaging strategies that handles the lack of a full classifier by using the restricted classifiers and reweighting them.  

\begin{algorithm}[t]
\caption{FedAvg with private label sets}
% label i, clients j, 
\KwData{Client label sets $\{\cY_k\}$ and reverse indices $\{I_k\}$}
\KwResult{Classifier $h(x) = \sigma(\theta_\psi^\top \phi(x))$ }
Initialize central parameters $\theta^0 = (\theta^0_\phi, \theta^0_\psi)$  \\
\For{each round $t = 0, ..., T-1$}{
    \For{each client $k = 1, ..., m$}{
        Distribute $(\theta^t_\phi, \theta^t_\psi[\cY_k])$ to client $k$ \\
        Receive client update $(\theta^t_{\phi,k}, \theta^t_{\psi,k})$ \\
    }
    $\theta_{\phi}^{t+1} = \sum_{k=1}^m\theta^t_{\phi,k} \frac{n_k}{n} \quad \mbox{where} \quad n = \sum_{k=1}^mn_k$ \\
    \For{each label $y \in \cY$}{
       $\theta_{\psi}^{t+1}(y) =\sum_{k:y\in \cY_k}\theta^t_{\psi,k}(I_k(y)) \frac{n_k}{n'_y} \quad \mbox{where} \quad n'_y = \sum_{k:y\in \cY_k}n_k$ \\
    }
}
Return classifier with parameters $\theta = (\theta_{\phi}^T, \theta_{\psi}^T)$ \\
\label{alg:restricted_fedavg}
% \caption{The algorithm to construct a multiclass classifier based on the restricted classifiers of each client.}
\end{algorithm}

\paragraph{Client-side modification} In each round $t$, each client $k$ is sent the full set of current encoder parameters $\theta^t_\phi$ and the subset of current classifier parameters corresponding to their label set, $\theta^t_{\psi}[\cY_k] \coloneqq [\theta^t_\psi(y) : y \in \cY_k]^\top \in \bbR^{|\cY_k|}$. Clients then proceed with local updates as normal. 

\paragraph{Server-side modification} In each round, $t$, the server receives parameter updates $(\theta^t_{\phi, k}, \theta^t_{\psi, k})$ from each client $k$ and averages the classifier parameters for each label $y$ based on the clients which have the label in their label set, weighted according to their sample size (see Algorithm \ref{alg:restricted_fedavg}). Encoder parameter updates $\theta_{\phi, k}^t$ are averaged as normal. 

Surprisingly, this simple method is well-justified under the softmax classifier model, provided that the classifier $h(x) = \sigma(\theta_\psi^\top \phi(x))$ is well-specified and clients' conditional label distributions (mechanisms) are what we call \emph{subset consistent}.
\begin{thmasmp}[Subset-consistent labeling mechanisms]\label{asmp:subset_covshift}
The labeling mechanisms of clients $k=1, .., n$, each with a distributions $p_k(\bX, Y)$ on a label set $\cY_k$, are \emph{subset-consistent} if the target label distribution $p(\bX, Y)$ satisfies
$$
\forall k, \bx : p_k(Y=y \mid Y\in \cY_k, \bX=\bx) = p(Y=y \mid Y\in \cY_k, \bX=\bx)~.
$$    
\end{thmasmp}
%\hltodo{Comment on this?}

Now, suppose that $h_\theta(x)$ is well-specified for the true labeling function $p(Y\mid X)$ given an optimal encoder $\phi$, that is, there are parameters $\theta_\psi$ such that
$$
p(Y=y \mid X=x) = \frac{e^{-\theta_\psi(y)^\top \phi(x)}}{\sum_{y'}e^{-\theta_\psi(y')^\top \phi(x)}} = \sigma(\theta_\psi^\top \phi(x))_y~.
$$
Then, the subset-conditional outcome can be parameterized as a softmax classifiers with parameters $\theta_\psi[\cY_k]$, the subset of $\theta_\psi$ restricted to $\cY_k$,
\begin{align}
p(Y=y \mid X=x, Y \in \cY_k) & = \frac{p(Y=y \mid X=x)}{\sum_{y' \in \cY_k} p(Y=y' \mid X=x)} = \frac{e^{-\theta_y^\top \phi(x)}}{\sum_{y' \in \cY_k} e^{-\theta_{y'}^\top \phi(x)}} \nonumber \\
& = \sigma(\theta_\psi[\cY_k]^\top \phi(x))_y, \label{eq:fedavg_equiv}
\end{align}
since the normalization terms over the full label set cancel. As a result, the optimal model in this circumstance has the same parameters $\theta(y)$ both centrally and in all clients $k$ with $y\in \cY_k$. Consequently, given an optimal encoder $\phi(x)$ in the sense above, \emph{any} convex combination of unbiased estimates $\hat{\theta}_{\psi,k}$ of client-optimal parameters is unbiased for the server-optimal parameters $\theta$. Client weighting based on sample size (as in Algorithm~\ref{alg:restricted_fedavg}) achieves the largest effective sample size (smallest variance)~\citep{zec2024overcoming}.

\paragraph{Remark.} In the deterministic case, where $\forall \bx, \exists y^* : p(Y=y^* \mid \bX=\bx) = 1$, Assumption~\ref{asmp:subset_covshift} corresponds to the often-used \emph{covariate shift} assumption~\citep{shimodaira2000improving} since the event $Y\in \cY_k$ does not alter the distribution of $Y$ for a given on $\bx$. In this case, aggregating \emph{perfect} client models $h_k(y\mid x)$ is trivial for a given $x$, since all of them will return $1$ for the correct label. In general, for stochastic labels $p_k(y\mid Y\in \cY_k, \bx) \neq p_l(y\mid Y\in \cY_l, \bx)$ for $\cY_k \neq \cY_l$. In either case, Assumption~\ref{asmp:subset_covshift} allows both marginal distributions $p_k(\bX)$ and $p_k(Y)$ to vary with $k$.

Based on the simple modifications above, we can also adapt the FedProx~\cite{fedprox} algorithm and other centralized model-averaging strategies. For FedProx, we simply omit a comparison of the final layers in the regularization term on the client as their sizes do not match. 

The approach detailed in this section is by itself a viable method and will produce a classifier for all classes. However, when representations are not optimal for all clients at once, or when there isn't a single classifier that is optimal in all clients, the justification from \eqref{eq:fedavg_equiv} fails, and simply averaging client parameters. may not be the best strategy. We explore an alternative strategy next.

\subsection{Representation alignment by central tuning}
\label{sec:alignment}
The fundamental problem of federated learning is the aggregation of multiple client models into a single central model that is beneficial to the whole federation. The classical approach of parameter averaging, and its adaptation to private label sets above, is specific to a few model classes (e.g., neural networks) and poorly justified when the averaged representation is suboptimal for some clients. 
Stepping back, the aggregation problem may be viewed as a special case of classifier combination or couplings~\citep{wu2003probability,hastie1997classification}. Classifier combination methods were developed to combine several \emph{binary} classifiers, e.g., support vector machines, on different pairs of labels into a single multi-class classifier. Today, this technique is rare as multi-class classifiers are trained routinely using neural networks with softmax outputs or (ensembles of) decision trees. However, in federated learning with heterogenous and private label sets, we face the same problem again since no client nor the server has access to labeled data from all classes.

%Furthermore, in the private label set setting, averaging the predictions of each model for the whole label set is not possible as the probability of label $y \not\in \cY_k$ will be 0 for any client $k$. To address this issue we can perform an aggregation based on the model predictions restricted to the intersection of label sets.

% A related idea is to align the representations of clients by making use of a central dataset. The assumption of a labeled central dataset can be cumbersome, as it can be impractical or expensive to obtain. However, if we do not require labeled data it lessens this burden substantially.

% At the heart of federated learning is the problem of combining (partially trained) classifiers, updated by clients on their own private data. %Ty
% If the representations learned are not optimal for all clients at once we may consider to tune the representations centrally to be consistent with each other. Consider aggregating models $h_1, h_2$ produced by two clients such that $h_k(y\mid \bx) \approx p_k(Y=y \mid  Y\in \cY_k, \bX=\bx)$ for $k \in \{1,2\}$. Our goal is to aggregate these to a single model $h(y\mid \bx) \approx p(Y=y\mid \bX=\bx)$. Averaging the predictions of each model for the whole label set will not work since the probability of label $y \not\in \cY_k$ will be 0 for any client $k$. What about averaging predictions where the label sets intersect?
% \hltodo{Theoretical justification!?}

Traditionally, classifier combination operates on the classifier functions themselves not on their parameters. In our case, the classifiers are estimates of the conditional label probability $p_k(Y\mid \bX)$ specific to each client $k$ and their label sets $\cY_k$. It is appropriate to ask whether there exists a perfect combination of perfect client classifiers, one that yields minimal error on the target distribution $p(\bX, Y)$. To understand this, we draw inspiration from the binary-to-multi-class problem of classifier combination~\citep{hastie1997classification} and note an important distinction to our setting: usually, classifier combination applies to multiple classifiers trained on different subsets of the same data, or at least on data from the same distribution. This implies a structure between the probability distributions that the classifiers aim to fit. For example, with $\cY=\{0,1,2\}$, a binary classifier can be used to distinguish classes $0$ and $1$ by training on samples $(\bx,y)$ labeled with $y\in \{0,1\}$. By design, $p(Y=1 \mid Y \in \{0,1\}, \bX=\bx) = p(Y=1\mid \bX=\bx)/p(Y\in \{0,1\} \mid \bX=\bx)$. 

In general, the clients in federated learning may have completely unrelated label distributions. However, if we suppose again that Assumption \ref{asmp:subset_covshift} holds, a perfect combination of perfect classifiers may be found.
%\hltodo{Does this make distributional shifts trivial? Can we still have $p_k(X) \neq p(X \mid Y \in \cC_k)$?}
%\hltodo{Weave after assumption move}
\begin{thmprop}[Perfect classifier combination]\label{prop:perfect_combo}%
Let Assumption~\ref{asmp:subset_covshift} hold for a set of clients $k=1, ..., m$ such that clients jointly cover all labels, $\cup_{k=1}^m \cY_k = \cY$. Then, the perfect central classifier $p(Y=y \mid \bX=\bx)$ can be aggregated from perfect client classifiers $\{p_k(Y=y \mid \bX=\bx)\}_{k=1}^m$.
\end{thmprop}
\begin{proof}
By Assumption~\ref{asmp:subset_covshift}, for all clients $k$, inputs $\bx \in \cX$, and outputs $y \in \cY_k$,
\begin{align}
p_k(y \mid Y\in \cY_k, \bx) = \frac{p(y \mid \bx)}{p(Y \in \cY_k \mid \bx)} 
= \frac{p(y \mid \bx)}{\sum_{y'\in \cY_k}p(y' \mid \bx)}~. \label{eq:perfect_combo}
\end{align}
In other words, $\forall \bx,y, \exists k : p(y \mid \bx) = c(\bx) p_k(y \mid Y\in \cY_k, \bx)$ with $c(\bx)$ a normalizing constant.
\end{proof}
The result for softmax classifiers in \eqref{eq:fedavg_equiv} is a special case of this result.

In practice, of course, we cannot expect to have perfect models of each client to combine---especially not \emph{during} federated learning. However, Proposition~\ref{prop:perfect_combo} gives direction for what a good aggregated model should satisfy. Consider an estimated client model $h_k(y\mid \bx) \approx p_k(y \mid Y\in \cY_k, \bx)$ and a good central model $h(y \mid \bx) \approx p(y \mid \bx)$. By \eqref{eq:perfect_combo}, it should hold that,
$$
\forall k\in [m], y,y' \in \cY_k : h_k(y\mid \bx)h(y'\mid \bx) \approx h_k(y'\mid \bx)h(y\mid \bx)~.
$$
This is a generalization of the argument in \cite{wu2003probability} to multi-class subset classifiers. We may use this to construct an aggregation criterion for the central model $h$, given a set of client models $\{h_k(y \mid \cY_k, \bx)\}_{k=1}^m$, first for a fixed input $\bx$,
\begin{equation}\label{eq:obj_fix_x}
\underset{h_\bx \in \Delta_{\cY}}{\text{minimize}}\; \sum_{k=1}^m w_k \sum_{\substack{y, y' \in \cY_k\\ y\neq y'}} (h_k(y\mid \bx)h_\bx(y') - h_k(y'\mid \bx)h_\bx(y))^2 ~.
\end{equation}%\hltodo{Remind about client weights $w_k$?}
If all client models are perfect, the minimizer of \eqref{eq:obj_fix_x} is a perfect central model at $\bx$ under the conditions of Proposition~\ref{prop:perfect_combo}. For high-dimensional or continuous $\bx$, it is not feasible to fit a separate central classifier to each possible input. Instead, we may use function approximation by fitting a classifier $h(y\mid \bX)$ from a class $\cH \subset \{h : \cX \rightarrow \Delta_\cY\}$ to minimize the expected error over $p(\bX)$.

\begin{equation}\label{eq:obj_exp}
\underset{h \in \cH}{\text{minimize}}\; \sum_{k=1}^m w_k \sum_{\substack{y, y' \in \cY_k\\ y\neq y'}} \E_\bX\left[ \big(h_k(y\mid \bX)h(y' \mid \bX) - h_k(y'\mid \bX)h(y \mid \bX)\big)^2 \right]
\end{equation}

We call this the \emph{pairwise} tuning loss and will use this as one of our objectives when combining classifiers centrally. In practice, the marginal distribution $p(\bX)$ is unknown and the expectation is intractable to compute, so we must solve $\eqref{eq:obj_exp}$ with respect to the empirical expectation $\hat{\E}[\bX]$ over a sample of data. Consequently, to use this method, we require that the central server has access to an \emph{unlabeled} data set of points $\bx_1, ..., \bx_m$ drawn from $p(\bX)$. Since access to tuning data is not required by methods based on parameter averaging, we must bear that in mind when comparing the empirical performance of the two approaches. 

For additional comparison, we also consider tuning-based classifier combination using the direct \emph{MSE} loss used in~\cite{wu2003probability}, 
\begin{equation}\label{eq:obj_MSE}
\underset{h \in \cH}{\text{minimize}}\; \sum_{k=1}^m w_k \sum_{\substack{y \in \cY_k}} \E_\bX\left[ \big(h_k(y\mid \bX) - h(y\mid \bX)\big)^2 \right] ~.
\end{equation}
In summary, we solve one of the two optimization problems above at each update to tune the classifier to be more aligned with the predictions of the client models. The tuned classifier is sent back to the clients and training proceeds as normal. 

\section{Experiments}
\label{sec:experiments}
% We perform several experiments on two benchmark data sets to evaluate the effects of label sparsity in FL. We consider both public and private label sets.
% \subsection{Experimental Setup}
%%% start writing this part
We use the well-known datasets CIFAR-10~\citep{krizhevsky2009learning} and Fashion-MNIST~\citep{xiao2017fashion} for constructing our experiments. We perform ablations where we vary the number of labels that a client has access to from 2-10. 
This entails choosing a random set of labels for each client which they then get distributed from the dataset equally. This means that, absent further intervention, the clients will not have an identical amount of labeled examples across the ablation points. To control for this, we perform a subsampling step where we subsample the client dataset randomly to consistently have 2000 samples in each client. When evaluating the impact of tuning on an unlabeled dataset centrally (Section~\ref{sec:alignment}), we use an unlabeled dataset with 5000 samples for CIFAR10 and 6000 for FashionMNIST. We use the standard test set splits for both datasets, both have 10000 samples.%The test sets are the standard splits provided for both datasets, and both have 10000 samples.

In the public label set setting, we use FedAvg, FedProx~\cite{fedprox} and FedRS~\cite{li2021fedrs} as baselines. In the private setting, we adapt FedAvg and FedProx to compare this approach to the central tuning (see Sections \ref{sec:private}--\ref{sec:alignment} for further details). As FedRS depends on knowledge of the full label set on the clients, we cannot use this method in the private setting.

 When performing central tuning, we train the server classifier for 3 epochs using one of two loss functions (Pairwise or MSE) in equations \eqref{eq:obj_exp} and \eqref{eq:obj_MSE}, respectively. The model aggregation then follows that of FedAvg (or our adaptation of FedAvg in the private setting).
We aggregate the results for different labels per client over 10 independent random seeds and the error bars denote a $95\%$ bootstrapped confidence interval over these splits. For the ablation over epochs per client, we aggregate over 3 random seeds. Further details, including the choice of hyperparameters for each method, can be found in Appendix \ref{app: exp_details}.
\subsection{Experimental Results}
We present the detailed results of our experiments below. For each method, we show the test accuracy of the model snapshot that achieves the highest validation accuracy during a run and then aggregate this across several seeds. More results, including an ablation over client epochs, can be found in Appendix \ref{app:additionalresults}. \newline

%\hltodo{Comment on performance going down with no labels not being due to decreasing sample size since that is fixed}

\textbf{CIFAR10:}
We present the results varying the amount of labels per client in Figure~\ref{fig:cifar_labelsperclient} with the specific case of 5 labels per client being shown in Figure \ref{fig:setup}b. We clearly see performance decreasing with decreasing number of labels. This effect is not due to decreasing sample size as that is fixed in the experiments. In the public label setting we see that FedRS performs the best while the tuning approach with the pairwise loss performs better than FedAvg and FedProx. Tuning with MSE loss seems to struggle here while, in the private setting, it performs the best. In the private setting, we can see that the tuning approach is superior to the adapted methods, although their variance is higher. Moreover, the MSE loss performs slightly better than the pairwise loss. Interestingly, the pairwise tuning seems to outperform FedAvg and FedProx in the public setting suggesting that the tuning of the representation can be of use in this setting also. This is likely because representational (mis)alignment can be an issue for federated learning whether labels are private or public. It is noteworthy that the adaptation of FedAvg and FedProx seem to exhibit a surprising robustness against the challenges of the private label sets, since they aggregate models from clients with disparate label sets. However, as the clients have an identical amount of labels, each of the feature extractors are trained to output the same label amounts which could help explain the robustness.%% Why? %% maybe refer to over communication rounds plot here....
\begin{figure}[t!]
    \centering
    \includegraphics[width=\textwidth]{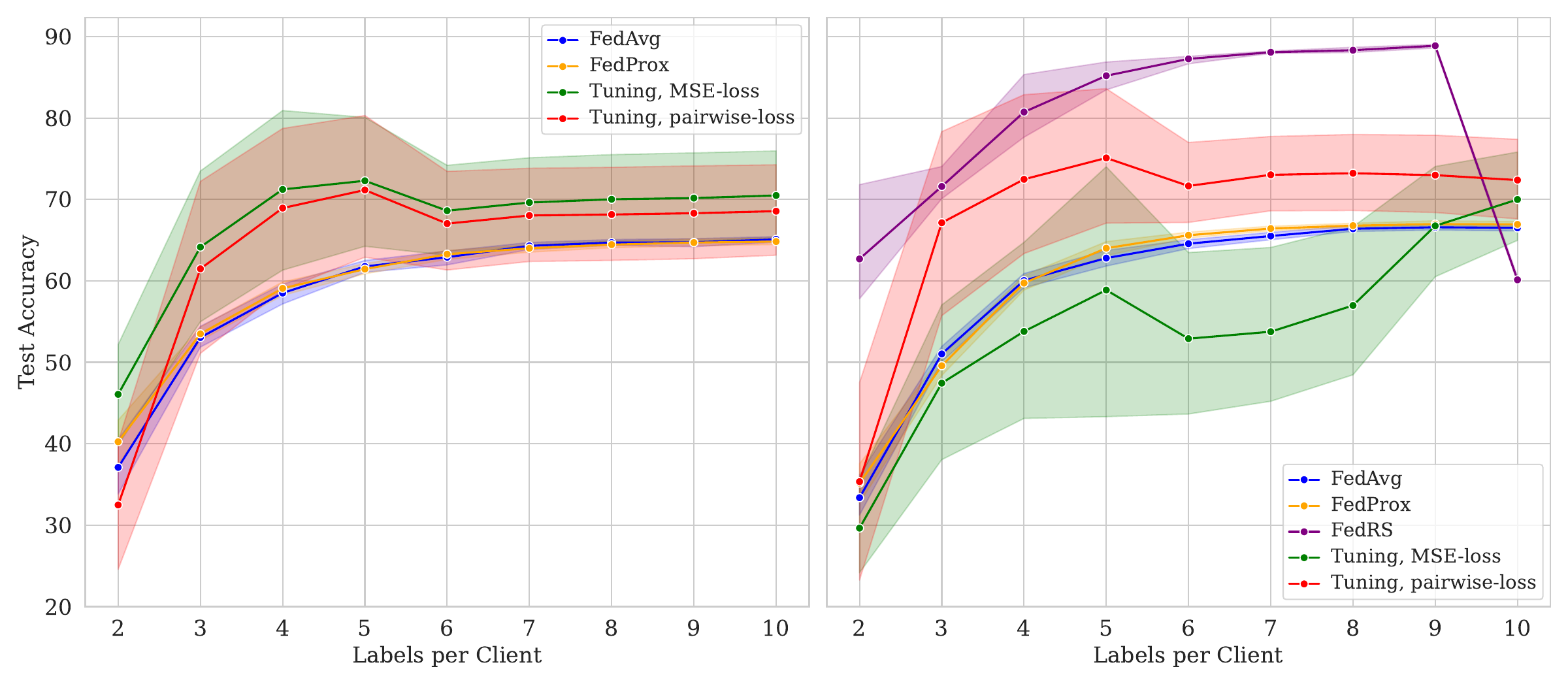}
    \caption{The performance of the methods in both the private (left) and public (right) settings on CIFAR10. Note that in the regular setting the pairwise loss performs better than FedAvg and the MSE loss while in the private setting the relationship is reversed.}
    \label{fig:cifar_labelsperclient}
\end{figure}

\textbf{Fashion-MNIST:}
As we can see in Figure~\ref{fig:fashion_labelsperclient}, the tuning methods do not outperform the adapted methods in the private setting on Fashion-MNIST. However, their large variance suggests that with more careful training they might perform at least equivalently. This may be due to the task being simple and an alignment of classifiers is unnecessary. We show results for 3 labels per client in Figure \ref{fig:fashion_3label}, where we see that the methods perform similarly in the private setting. %

The relationship that the pairwise loss performs better than the MSE loss in the public setting and worse in the private setting holds true here also. This could be due to the pairwise loss having cases where the loss is large due to lack of labels in clients. See Appendix \ref{app: pairwise_issue} for a discussion of this issue. We see a robustness of FedAvg and Fedprox to the private label sets here as well. This indicates that the adapted methods are a pragmatic choice that also works well in this restricted setting.

In the public setting, we see that the tuning methods underperform the other baselines, with FedAvg, FedProx, and FedRS all performing about equally well. This suggests that there is either limited misalignment between client representations or that it does not adversely affect performance. Instead, the central tuning seems to interfere with the successful learning during federation in the public setting, possibly due to increased variance in the central model. We can observe this variance in the figures over communication rounds in Appendix \ref{app:per_round}
\begin{figure}[t!]
    \centering
    \includegraphics[width=\textwidth]{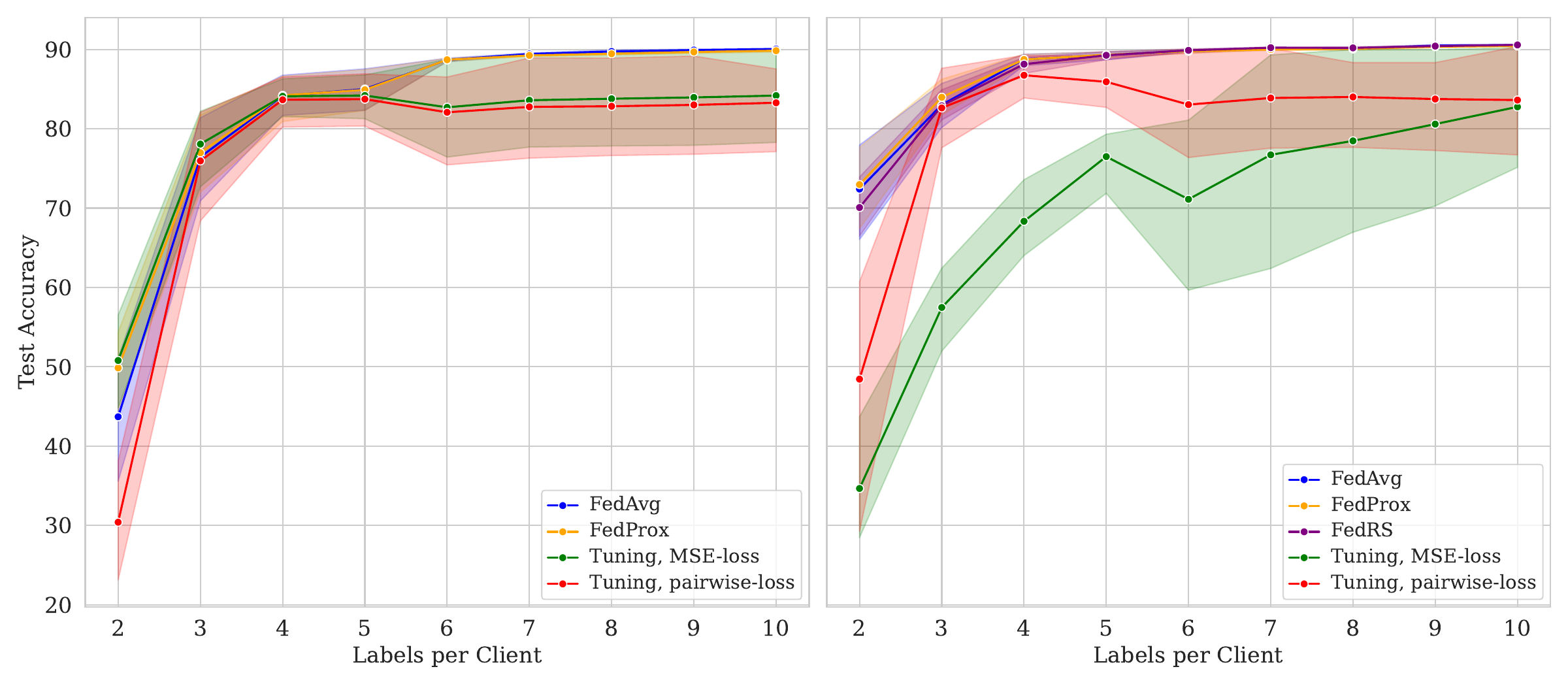}
    \caption{The performance of the methods in both the private and public settings on FashionMNIST. We note that the tuning approaches do not outperform the adapted methods in the private setting.}
    \label{fig:fashion_labelsperclient}
\end{figure}
\begin{figure}
    \centering
    \includegraphics[width=0.55\textwidth]{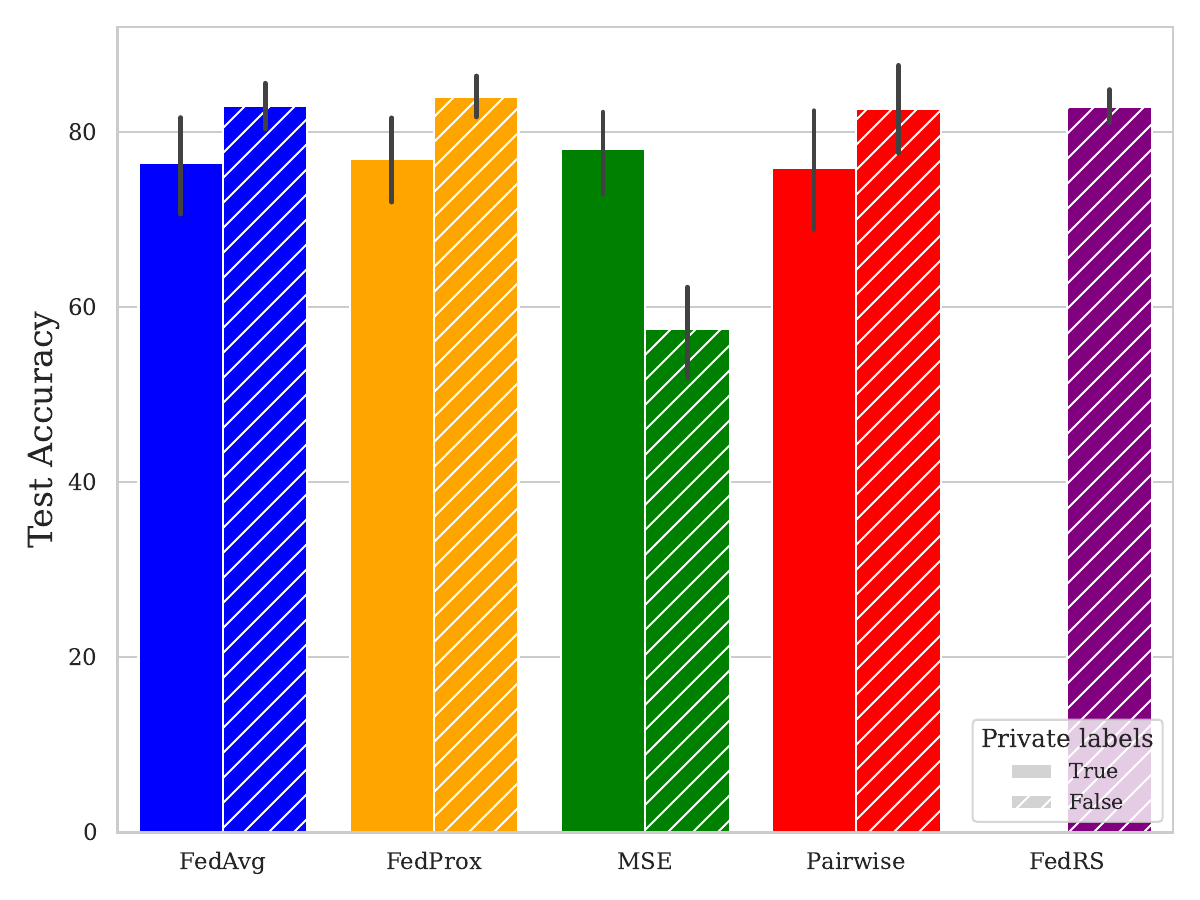}
    \caption{The results for the private and public settings for 3 labels per client on the FashionMNIST dataset.}
    \label{fig:fashion_3label}
\end{figure}

\section{Discussion}
This work investigated the impact of label set heterogeneity on Federated Learning, with a specific focus on the challenging and practical \textit{private label set} setting where clients are unaware of the global label space. Our experiments reveal several key insights into the behavior of both standard and adapted FL algorithms under these conditions.  

A primary finding is the surprising robustness of our adaptation of FedAvg to private label sets. We hypothesize that even when individual clients train on a small subset of labels, the shared feature extractor learns a common, semantically rich representation space. The simple aggregation strategy for the classifier weights (Algorithm 1), which combines knowledge on a per-class basis, proves to be a powerful and efficient method for stitching together these partial views into a coherent whole. This establishes adapted FedAvg as a formidable baseline, suggesting that complex alignment mechanisms may not always be necessary if there is sufficient label overlap across the client population.

Further, the tuning approaches do seem to work well for the private setting in some cases while performing worse than adapted methods in others. We also observe that the tuning yields better performance when the sparsity is more extreme. This may be due to the alignment problem being harder with an increasing amount of labels. Notably, we observed a performance reversal between the two tuning losses. In the public setting, the pairwise loss was superior, whereas in the more challenging private setting, the MSE loss performed better. We attribute this phenomenon to a critical vulnerability in the pairwise loss, detailed in Appendix B. When a label is globally absent from all participating clients in a round (a scenario far more likely with fewer labels per client) the pairwise loss generates problematic gradients by comparing against a class for which no client has information. The MSE loss, by directly comparing the global model's predictions to each client's predictions on their known classes, appears more resilient to this issue. It provides a more stable, albeit less constrained, learning signal in the face of extreme sparsity.

A key limitation of the tuning methods are the fact that solving the optimization problem for each communication round could become difficult computationally as the amount of labels, and the number of clients, increases. Also, the existence of unlabeled data is an additional burden to bear that may be impractical in application. The adapted methods do not share these limitations and could be an alternative if tuning methods are computationally infeasible or if there does not exist an unlabeled dataset centrally.

In future work, more realistic datasets could be considered which naturally exhibit label set heterogeneity. In addition, some other FL methods could perhaps be adapted to the private label set setting. Moreover, there could be further consideration of and comparison with other tuning losses. 
%Furthermore, the central tuning approach seems perform better in cases with more extreme sparsity, and then the MSE loss seems better suited than the pairwise loss. We conjecture that this may be related to the issue of a missing label in all client dataset, which can adversely impact the pairwise loss (see Appendix \ref{app: pairwise_issue}). 

\bibliographystyle{plainnat}
\bibliography{references}

%\end{thebibliography}
\begin{appendix}
\section{Experimental details}
\label{app: exp_details}
Here we detail the specifics of our experiments. In general we do a 80/20 train and validation split. For the tuning methods we train the classifier for three epochs over the unlabeled dataset each round. Both the central tuning and the client training uses the Adam optimizer with identical learning rates. 

\subsection{Model}
We use a simple CNN model for both experiments. 
\begin{table}[ht]
    \centering
    \begin{tabular}{c|c}
        Layer & \# Filters \\
        \hline
         Convolution& 3 \\
         Convolution & 3 \\
         $2\times 2$ Max pooling & - \\
         Convolution& 3 \\
         Convolution & 3 \\
         $2\times 2$ Max pooling & -\\
         Flatten & -\\
         Fully connected & - \\
         Dropout (0.5) &
    \end{tabular}
    \caption{Layers of the CNN used in the experiments.}
    \label{tab:cnn}
\end{table}

\subsection{Hyperparameters}
\begin{table}[ht]
    \centering
    \begin{tabular}{c|c}
         Learning rate & $1\times10^{-3}$ \\
         Batch size & 64\\
         Epochs per round & 1\\
         Number of communication rounds & 100\\
         FedProx $\mu$& $1\times10^{-2}$ \\
         FedRS $\alpha$ & 0.5 \\
    \end{tabular}
    \caption{hyperparameters used during training.}
    \label{tab:hyperparams}
\end{table}

\section{Effect of pairwise loss function with a missing label}
 \label{app: pairwise_issue}

Consider the loss function:
\[
\text{loss\_tensor} = (h_{k,y_j} \cdot p_i - h_{k,y_i} \cdot p_j)^2
\]
where:
\begin{itemize}
    \item \(h_{k,y_i}\) is the probability output by client model \(k\) for class \(i\) (mapped to the client's local label space).
    \item \(h_{k,y_j}\) is the probability output by client model \(k\) for class \(j\) (mapped to the client's local label space).
    \item \(p_i\) is the probability output by the global model for class \(i\).
    \item \(p_j\) is the probability output by the global model for class \(j\).
\end{itemize}

\textbf{Scenario:} A single label, label 9, is not present in any client's dataset.  All other labels (0-8) are present in at least one client.

\textbf{Consequences:}

\begin{enumerate}
    \item \textbf{Zero client probabilities for label 9:} Because no client has seen label 9, their models will output near-zero (or zero) probabilities for this label:
    \[
    h_{k,y_9} \approx 0 \quad \forall k
    \]

    \item \textbf{Problematic loss calculation when label 9 is involved:}  The loss calculation becomes problematic when either \(i=9\) or \(j=9\).  Let's analyze both cases:

    \begin{itemize}
        \item \textbf{Case 1: \(i = 9\)}
        \[
        \text{loss\_tensor} = (h_{k,y_j} \cdot p_9 - h_{k,y_9} \cdot p_j)^2 \approx (h_{k,y_j} \cdot p_9 - 0 \cdot p_j)^2 = (h_{k,y_j} \cdot p_9)^2
        \]
        The loss depends on the global model's probability for label 9 (\(p_9\)) and the client's probability for other labels \(j\).  The global model is penalized if \(p_9\) is non-zero, even though no client provides information about label 9.

        \item \textbf{Case 2: \(j = 9\)}
        \[
        \text{loss\_tensor} = (h_{k,y_9} \cdot p_i - h_{k,y_i} \cdot p_9)^2 \approx (0 \cdot p_i - h_{k,y_i} \cdot p_9)^2 = (h_{k,y_i} \cdot p_9)^2
        \]
       Similar to Case 1, the loss depends on \(p_9\) and client probabilities for other labels. The global model is penalized, and gradients related to label 9 are based on "noise" from the clients.

    \end{itemize}

    \item \textbf{Global model degradation:} The global model's representation for label 9 is negatively impacted. The loss pushes \(p_9\) towards zero because that's the only way to reduce the loss when paired with the near-zero client probabilities.  This harms the global model's ability to generalize, even for labels present in client data.

     \item \textbf{Unfair penalization:}  The global model receives gradients that are based on a comparison against the absent label, creating unstable behaviour.
     \end{enumerate}
     \section{Additional empirical results}
     \label{app:additionalresults}
     Here we show some additional results.
     \subsection{Test accuracy over communication rounds}
     \label{app:per_round}
     Here we present the test accuracy over time for the public and private settings. We compare FedAvg to the tuning methods we propose. We can see in Figures \ref{fig:fashion_over_private}---\ref{fig:cifar_over_public} that the convergence of the methods seem to occur at similar times in the training. To note is the increased variance in the tuning methods.
     \begin{figure}
         \centering
         \includegraphics[width=\linewidth]{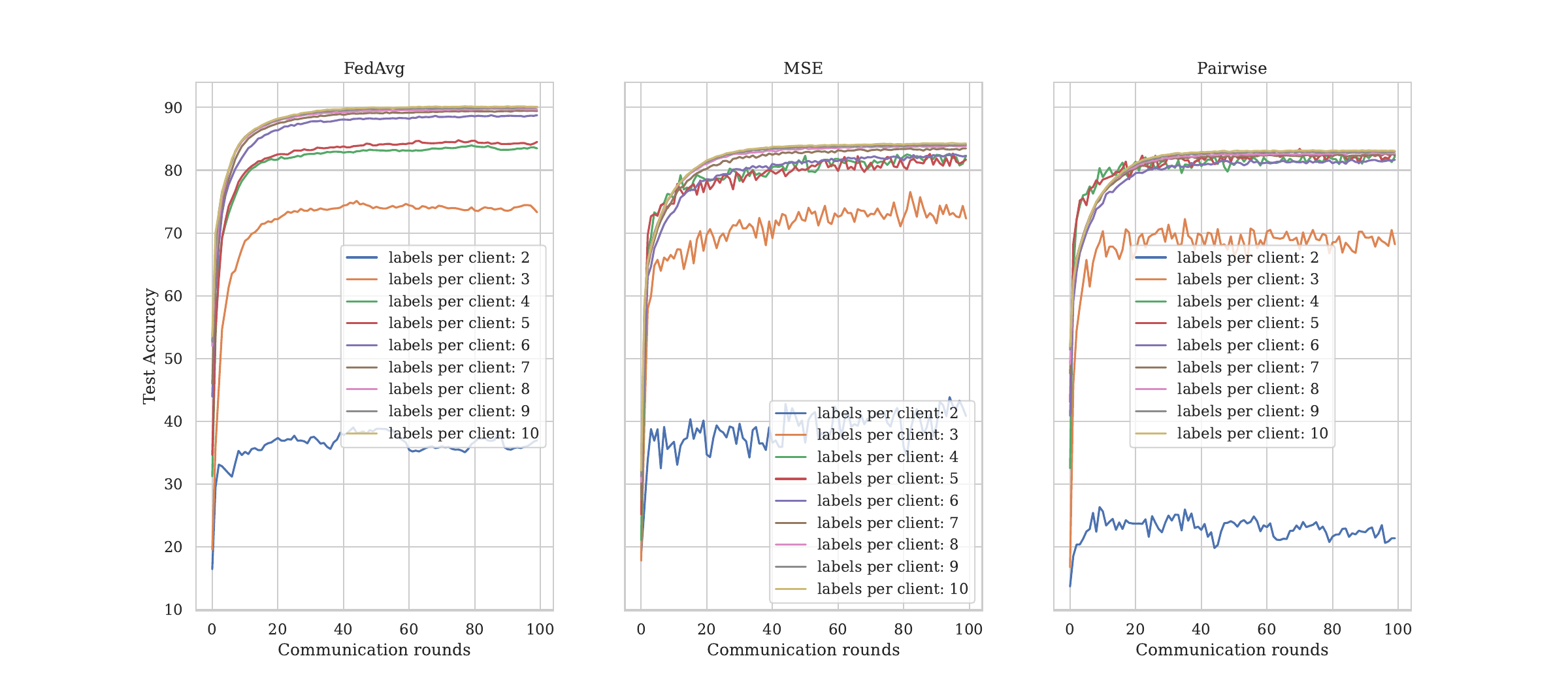}
         \caption{Test accuracy over rounds in the FashionMNIST task. The labels sets are private.}
         \label{fig:fashion_over_private}
     \end{figure}
       \begin{figure}
         \centering
         \includegraphics[width=\linewidth]{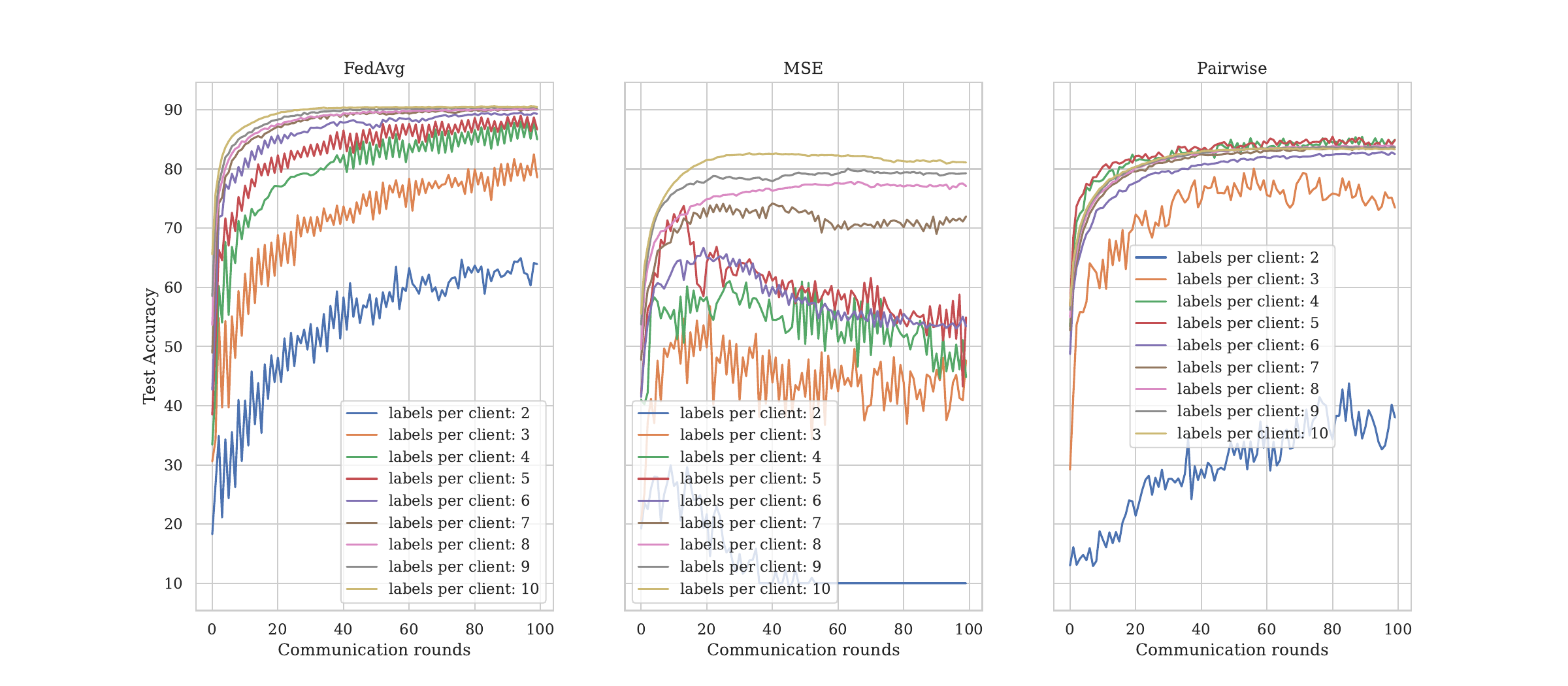}
         \caption{Test accuracy over rounds in the FashionMNIST task. The labels sets are public.}
         \label{fig:fashion_over_public}
     \end{figure}
       \begin{figure}
         \centering
         \includegraphics[width=\linewidth]{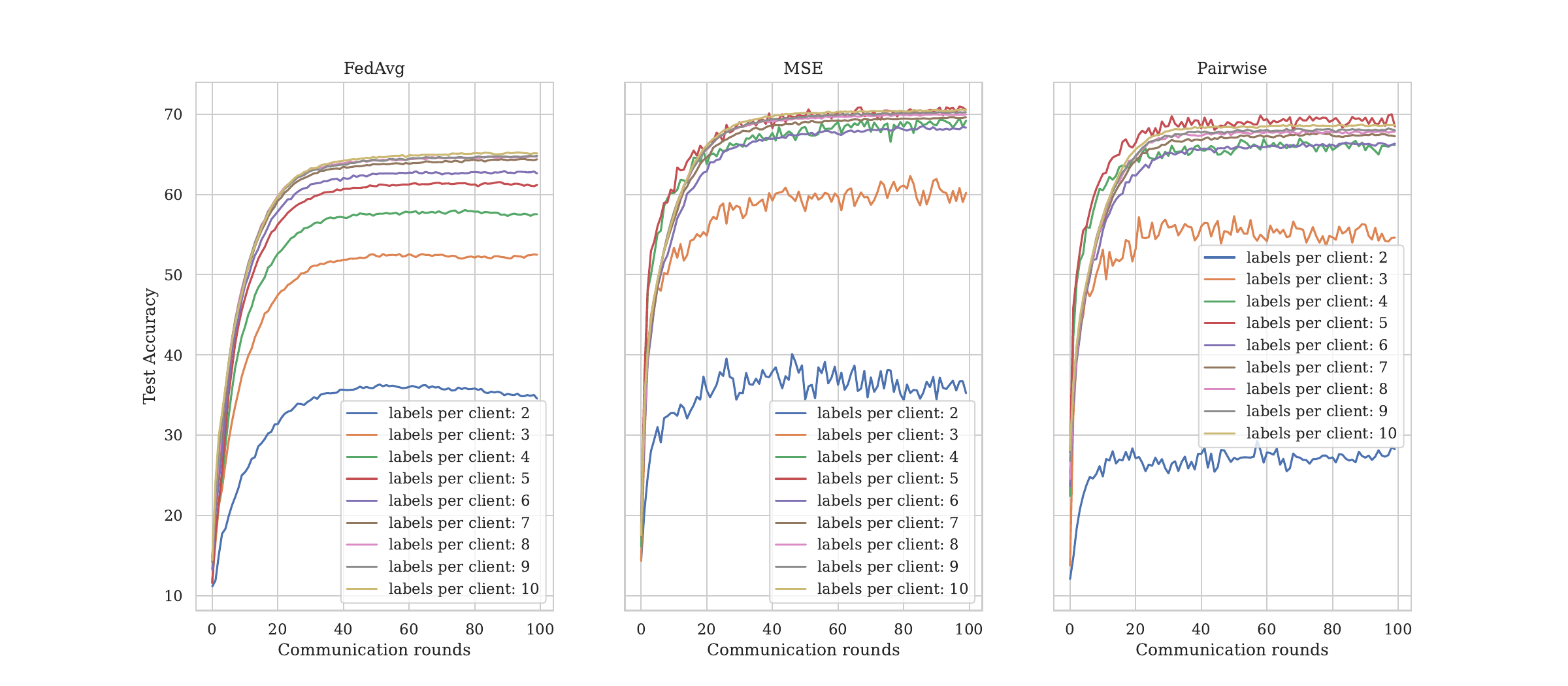}
         \caption{Test accuracy over rounds in the CIFAR10 task. The labels sets are private.}
         \label{fig:cifar_over_private}
     \end{figure}
       \begin{figure}
         \centering
         \includegraphics[width=\linewidth]{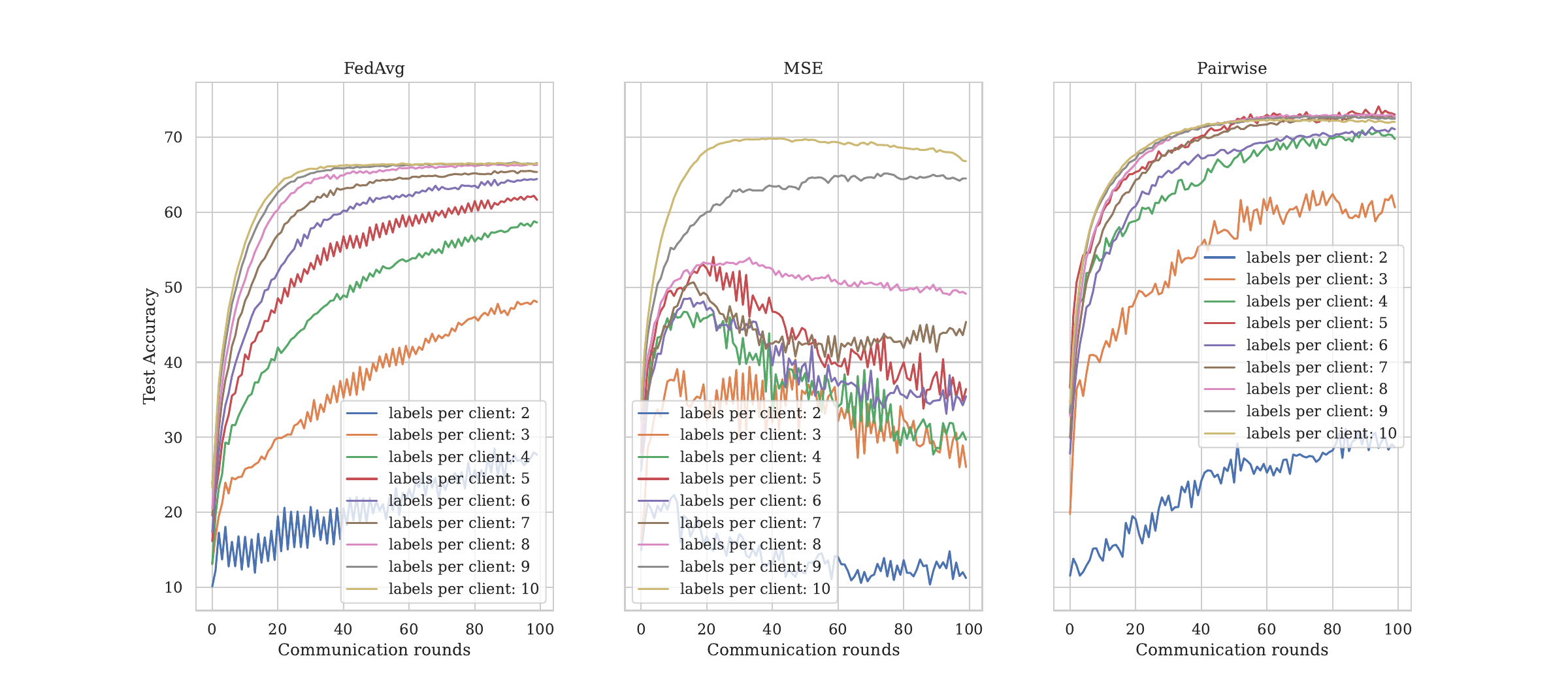}
         \caption{Test accuracy over rounds in the FashionMNIST task. The labels sets are Public.}
         \label{fig:cifar_over_public}
     \end{figure}
     \subsection{Labels per client barplots}
     Here are additional barplots to compare performance at different labels per client.
     \begin{figure}
    \centering
    \includegraphics[width=\textwidth]{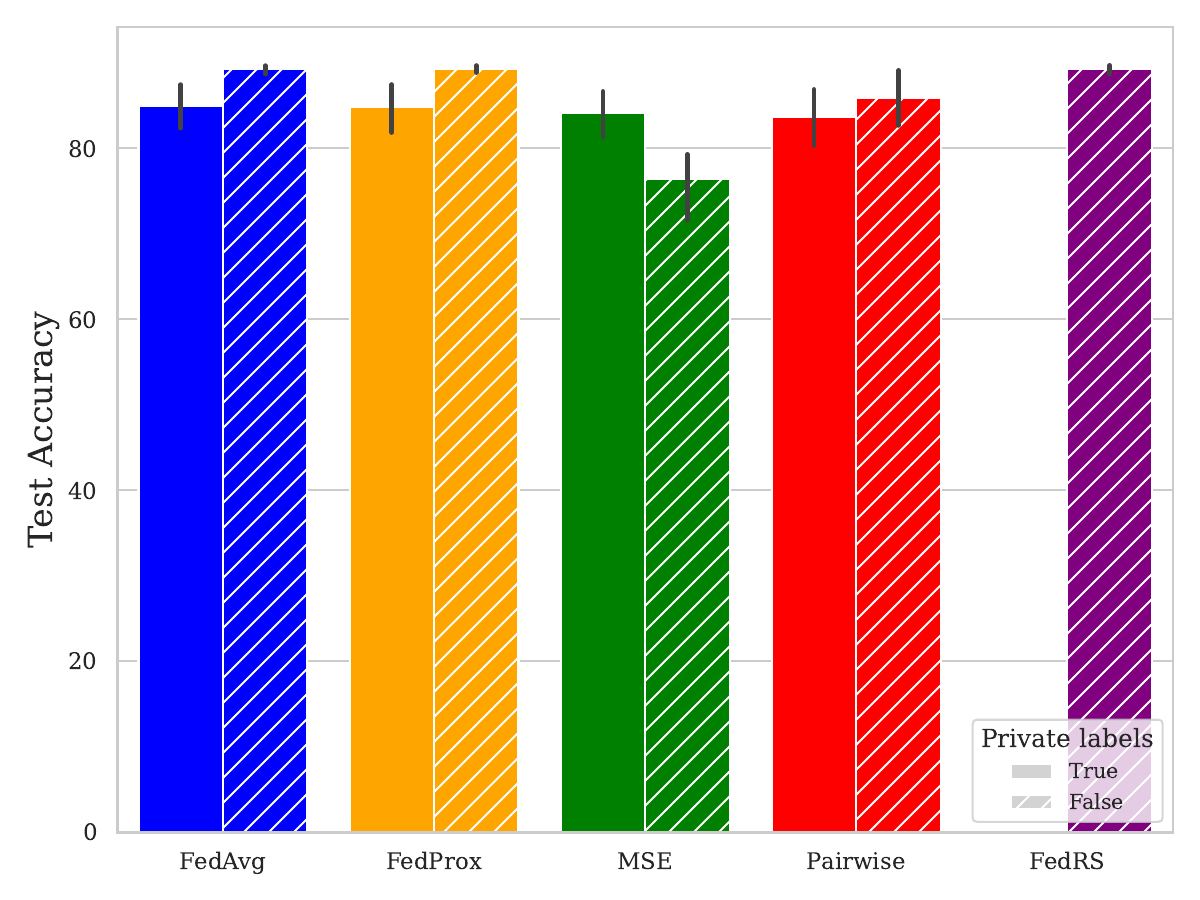}
    \caption{The results for the private and public settings for 5 labels per client on the FashionMNIST dataset.}
    \label{fig:fashion_5label}
\end{figure}
\begin{figure}
    \centering
    \includegraphics[width=0.8\textwidth]{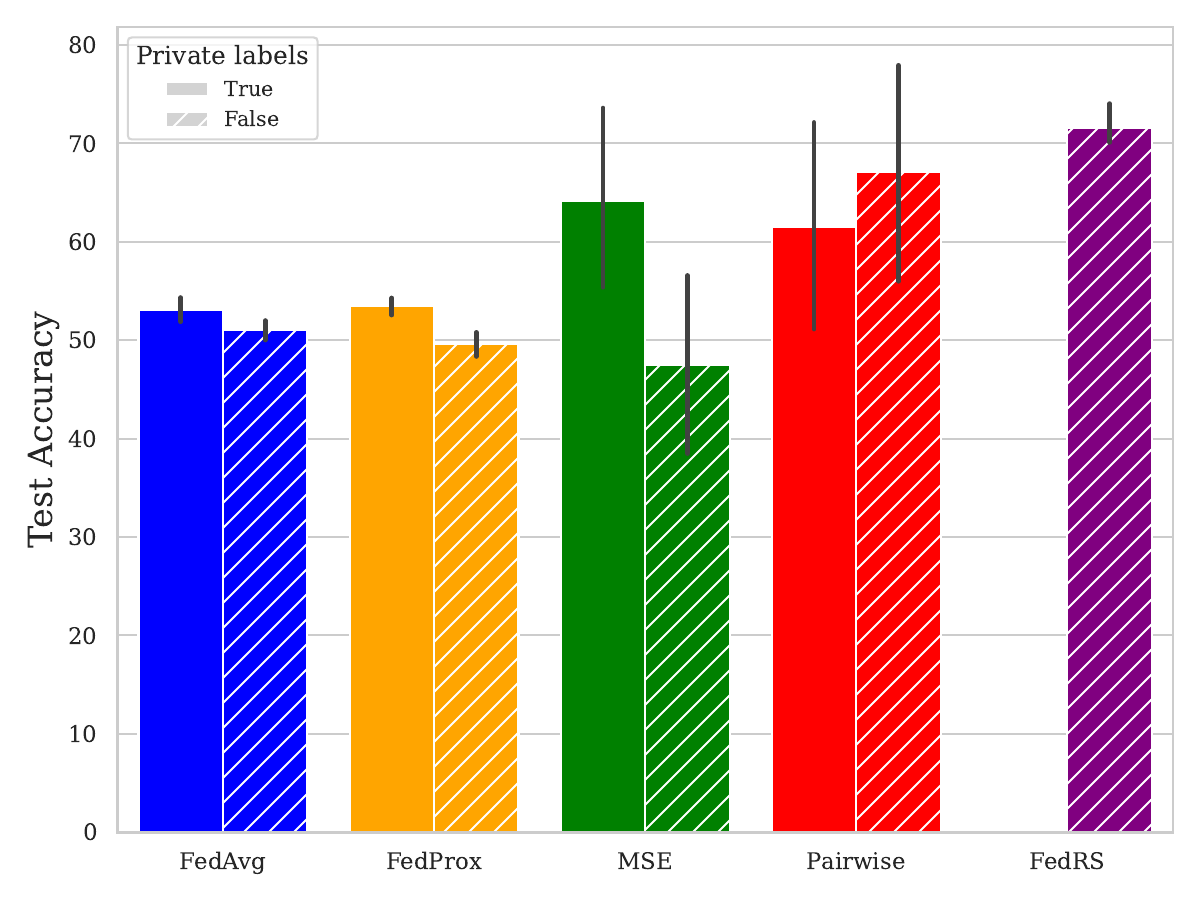}
    \caption{The results for the private and public settings with 3 labels per client on the CIFAR10 dataset.}
    \label{fig:cifar_3label}
\end{figure}

     \subsection{Local epochs}
     Here we present the results of experiments on CIFAR10 with varying labels per client in figures \ref{fig:cifar_localepochs}--\ref{fig:cifar_localepochs3}. Note, that the central tuning step is done here with a SGD optimizer and not Adam as in other experiments.
     We see that our ablations across different values of epochs per round do not seem to meaningfully impact performance. Furthermore, the same pattern of the MSE loss performing better in the private setting than the pairwise is seen here. 
     \begin{figure}
    \centering
    \includegraphics[width=\textwidth]{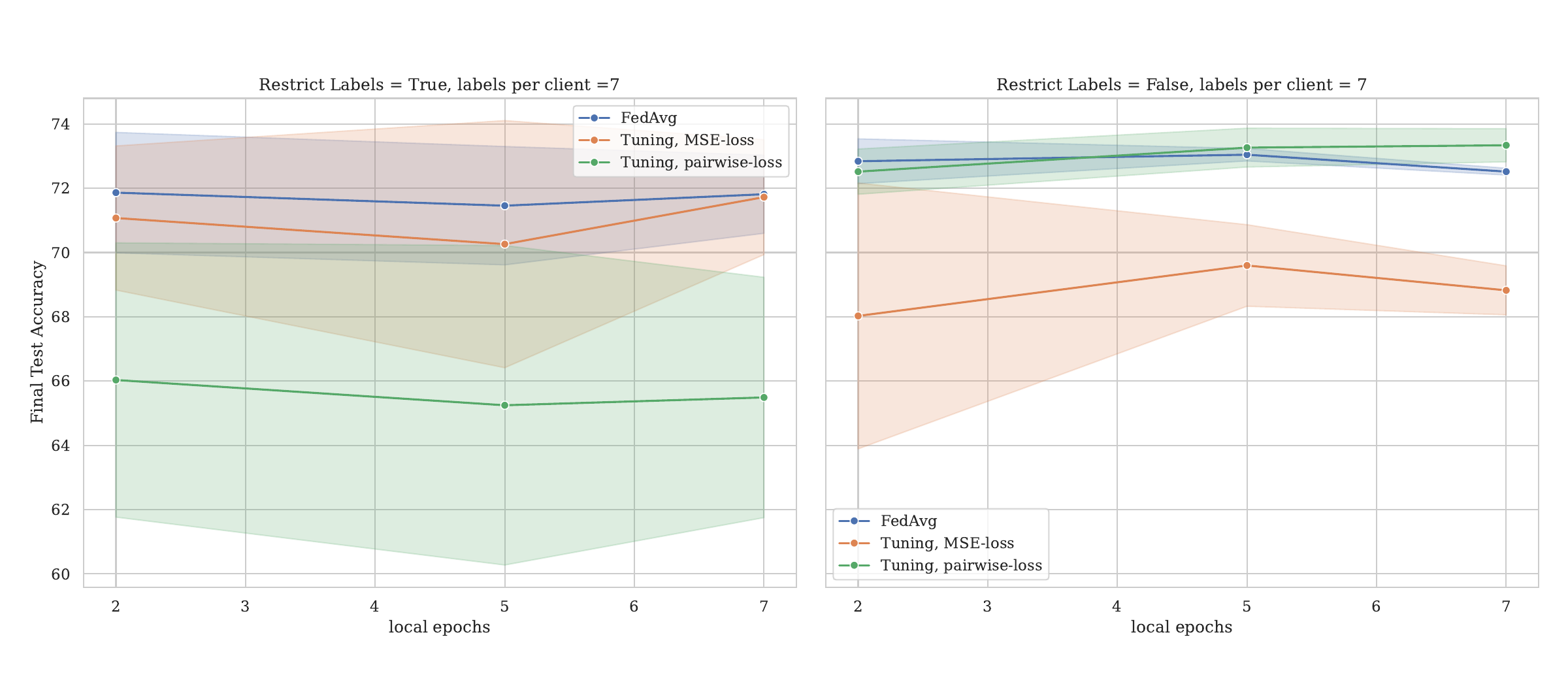}
    \caption{Ablation with differing amounts of local epochs on CIFAR10. We do not observe any particular drop in performance for either FedAvg or the tuning methods.}
    \label{fig:cifar_localepochs}
\end{figure}
     \begin{figure}
    \centering
    \includegraphics[width=\textwidth]{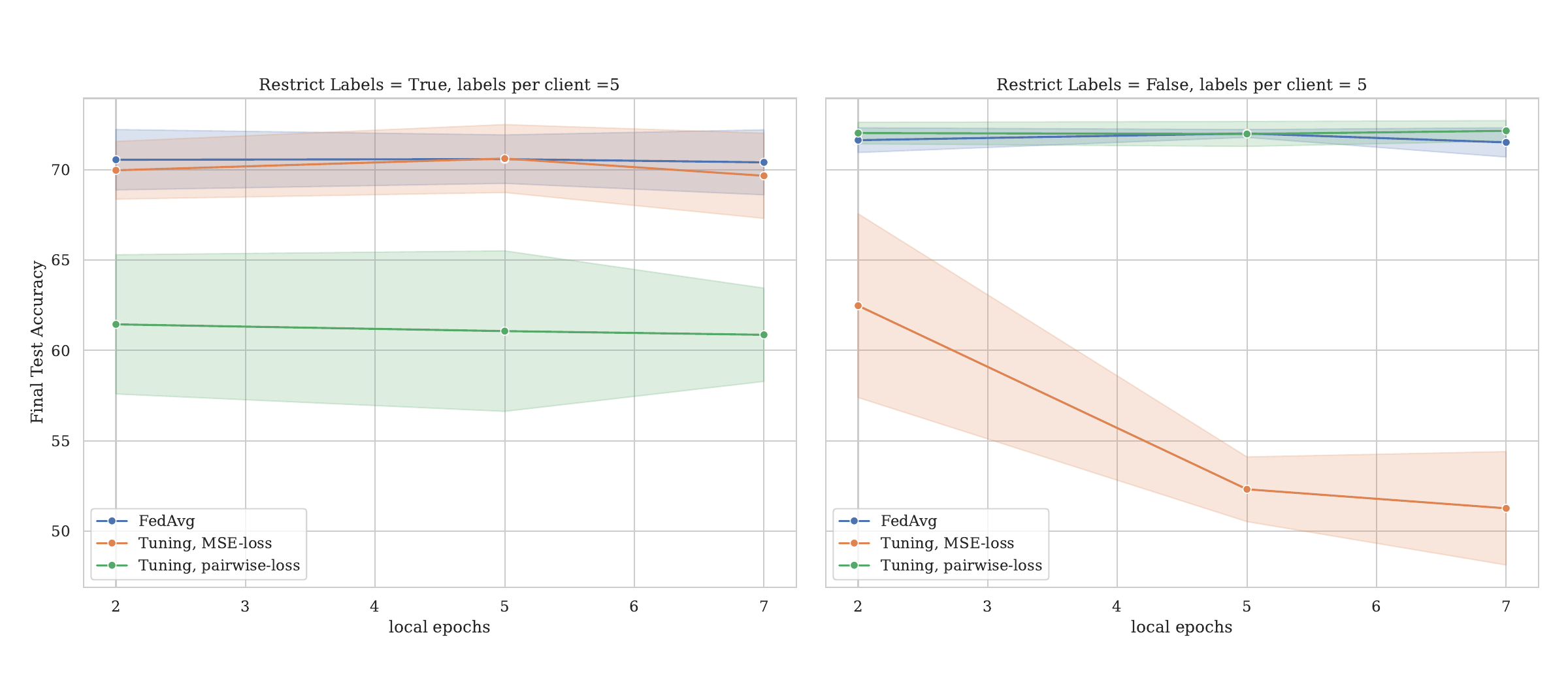}
    \caption{Comparison of different labels per client on CIFAR10 with E=5.}
    \label{fig:cifar_localepochs5}
\end{figure}
 \begin{figure}
    \centering
    \includegraphics[width=\textwidth]{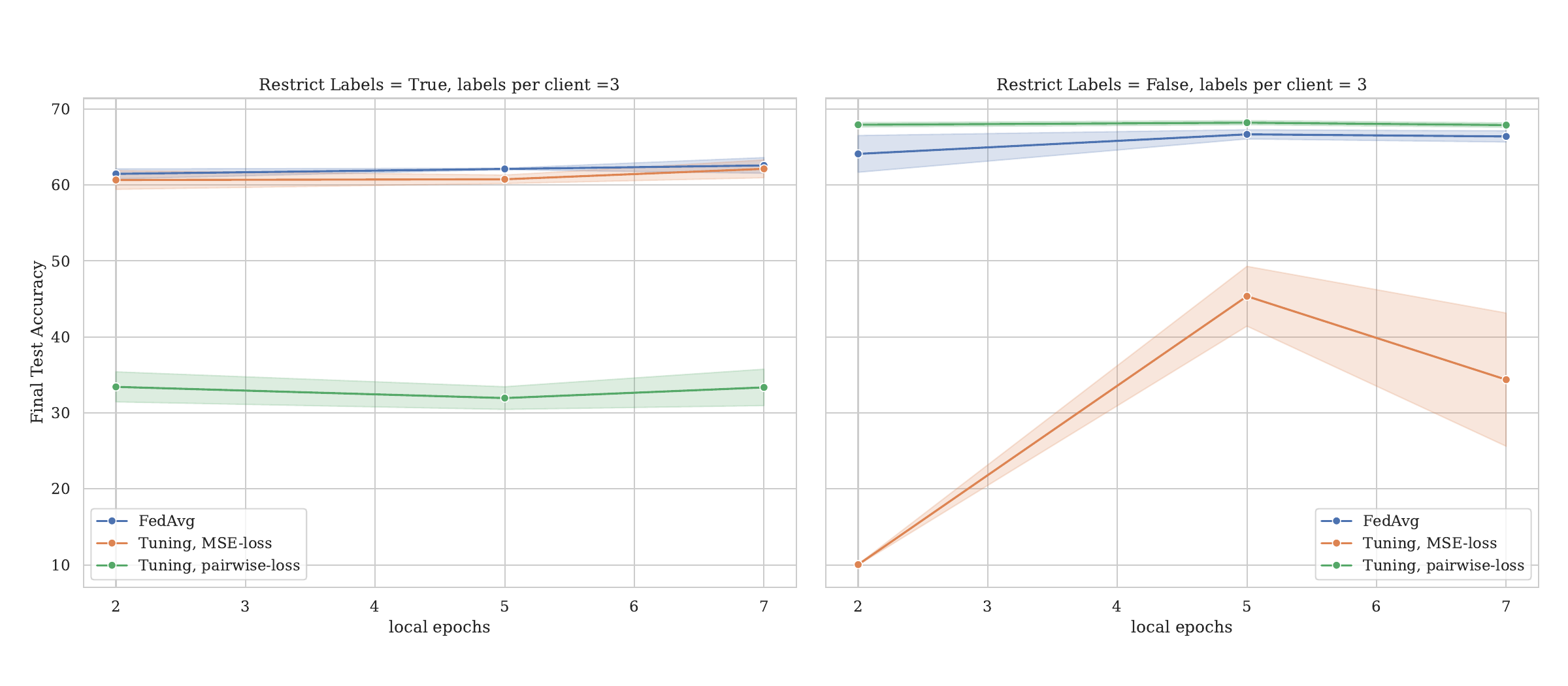}
    \caption{Comparison of different labels per client on CIFAR10 with E=3.}
    \label{fig:cifar_localepochs3}
\end{figure}
\end{appendix}
\end{document}